\documentclass[english]{article}
\usepackage[T1]{fontenc}
\usepackage[latin9]{inputenc}
\usepackage{amsmath}
\usepackage{amsthm}
\usepackage{amssymb}
\usepackage{graphicx}

\makeatletter

\providecommand{\tabularnewline}{\\}

\theoremstyle{plain}
\newtheorem{thm}{\protect\theoremname}[section]
  \theoremstyle{definition}
  \newtheorem{defn}[thm]{\protect\definitionname}
  \theoremstyle{plain}
  \newtheorem{prop}[thm]{\protect\propositionname}

\usepackage{spconf}

\title{Automatic Thresholding of SIFT Descriptors}
\name{Matthew R. Kirchner\thanks{This research was supported by the Office of Naval Research, ILIR: 4764}}
\address{Image and Signal Processing Branch, Research Office, Code 4F0000D \\ Naval Air Warfare Center Weapons Division, China Lake, CA 93555 USA \\ matthew.kirchner@navy.mil}

\makeatother

\usepackage{babel}
  \providecommand{\definitionname}{Definition}
  \providecommand{\propositionname}{Proposition}
\providecommand{\theoremname}{Theorem}

\begin{document}
\maketitle
\ninept  
\begin{abstract}
We introduce a method to perform automatic thresholding of SIFT descriptors
that improves matching performance by at least 15.9\% on the Oxford
image matching benchmark. The method uses a contrario methodology
to determine a unique bin magnitude threshold. This is done by building
a generative uniform background model for descriptors and determining
when bin magnitudes have reached a sufficient level. The presented
method, called meaningful clamping, contrasts from the current SIFT
implementation by efficiently computing a clamping threshold that
is unique for every descriptor.
\end{abstract}

\section{Introduction\label{sec:Introduction}}

The SIFT descriptor, introduced by Lowe in \cite{lowe2004distinctive},
is a widely used descriptor in image processing and machine learning.
The descriptor and its variants have been thoroughly studied and have
been shown to systematically outperform other descriptors \cite{mikolajczyk2005performance}.
Many extensions have been proposed, some of which include sampling
on a log-polar grid \cite{mikolajczyk2005performance,rabin2009approches},
reducing the dimension with PCA \cite{ke2004pca}, and scale pooling
by averaging descriptors sampled from a neighborhood around the detected
scale \cite{dongS15}. However, little known work has been performed
to study and enhance the descriptor thresholding that is presented
as part of the method. This thresholding \cite{otero2014anatomy},
also called \emph{clamping}, was introduced by Lowe as an ad-hoc way
to achieve robustness to non-linear illumination effects, such as
sensor saturation. This would lead us to believe the clamping process
would improve matching performance on image pairs that exhibit significant
illumination changes; but have little effect on images with similar
lighting conditions. However, Lowe's clamping method can greatly increase
matching performance (14.4\% improvement on the Oxford dataset) on
general image pairs even when no significant illumination changes
exist. 

This work proposes a novel method, which we call \emph{meaningful
clamping} (MC), to automatically threshold SIFT descriptors and improves
on the idea of clamping by providing a rigorous process to compute
the clamping threshold. This leads to significantly increased performance
over the existing clamping method on a wide variety of image matching
problems. The method is based on a contrario methodology for computing
detection thresholds \cite{desolneux2007gestalt}, and is introduced
in Sec. 3. Matching results with experiments performed on the Oxford
dataset \cite{mikolajczyk2005comparison} are shown in Sec. 5, and
confirm state-of-the-art results.

\section{The SIFT Descriptor\label{sec:The-SIFT-Descriptor}}

The image matching problem can be separated into two parts: feature
detection and feature description. The goal of a feature detector
is to produce a set of stable feature frames that can be detected
reliably across corresponding image pairs. Examples of methods that
detect similarity feature frames include SIFT, SURF \cite{bay2006surf},
SFOP \cite{forstner2009detecting}, Harris-Laplace \cite{mikolajczyk2004scale},
and Hessian-Laplace \cite{mikolajczyk2004scale}. Other methods have
been developed to detect affine feature frames such as MSER \cite{matas2004robust},
LLD \cite{cao2008theory}, Harris-Affine \cite{mikolajczyk2004scale},
and Hessian-Affine \cite{mikolajczyk2004scale}. For any given detected
feature, its frame determines how to sample a normalized patch $J\left(x,y\right)$,
for which we build a descriptor ${\bf d}$. \emph{The goal of the
descriptor is to distinctly represent the image content of the normalized
patch in a compact way}. 

We propose to create an extension of the SIFT descriptor, since it
has been shown to systematically outperform other descriptors \cite{mikolajczyk2005performance}.
The SIFT descriptor is a smoothed and weighted 3D histogram of gradient
orientations. For any patch $J$, we form a gradient vector field
$\nabla J$. We define the grid $\Lambda$, which determines the bin
centers $x_{i},y_{j},\theta_{k}$ of the histogram and has size $n(x)\times n(y)\times n(\theta)$.
In typical implementations, $\Lambda$ is chosen to have $4\times4$
spatial bins and $8$ angular bins. With ${\bf x}=\left(x,y\right)$
and $\ell=\left(i,j,k\right)\in\Lambda$, a single, pre-normalized
spatial bin of the SIFT descriptor can be written as the integral
expression: 
\begin{equation}
{\bf d}\left(\ell|J\right)=\int g_{\sigma}\left({\bf x}\right)w_{\alpha}\left(\angle\nabla J\left({\bf x}\right)\right)w_{ij}\left({\bf x}\right)\left\Vert \nabla J\left({\bf x}\right)\right\Vert d{\bf x},\label{eq: SIFT descriptor formal definition.}
\end{equation}
where $w_{ij}\left({\bf x}\right)=w\left(x-x_{i}\right)w\left(y-y_{j}\right)$
\cite{dongS15,vedaldi2010vlfeat}. The weight function $w_{ij}$ is
a bilinear interpolation with
\[
w\left(z\right)=\text{max}\left(0,1-\frac{n(z)}{2\lambda_{\text{patch}}}\left|z\right|\right);
\]
 and
\[
w_{\alpha}\left(\theta\right)=\text{max}\left(0,1-\frac{n(\theta)}{2\pi}\left|\theta_{k}-\theta\,\text{mod}\,2\pi\right|\right)
\]
 is an angular interpolation \cite{otero2014anatomy}. The parameter
$\lambda_{\text{patch}}$ is the radius of $J$ such that the patch
has dimensions $2\lambda_{\text{patch}}\times2\lambda_{\text{patch}}$.
The histogram samples are also weighted by a Gaussian density function
$g_{\sigma}\left({\bf x}\right)$, the purpose of which is to discount
the contribution of samples at the edge of the patch with the goal
to reduce boundary effects. The building of SIFT descriptors using
Eq. \ref{eq: SIFT descriptor formal definition.} for all experiments
was performed with the VLFeat open source vision library \cite{vedaldi2010vlfeat}\footnote{The VLFeat library estimates Eq. \ref{eq: SIFT descriptor formal definition.}
by sampling a discrete grid.}. For further details on how the descriptor was constructed, the reader
is encouraged to review \cite{otero2014anatomy,vedaldi2010vlfeat}.

\subsection{Clamping\label{subsec:Clamping}}

In an effort to design a descriptor to be robust to non-linear contrast
changes, Lowe proposed to threshold the bin magnitudes of the descriptor.
Lowe defines this threshold as
\begin{equation}
{\bf d}_{c}\left(\ell\right)=\text{min}\left({\bf d}\left(\ell\right),c\left\Vert {\bf d}\right\Vert \right),\label{eq: Clamping}
\end{equation}
with the parameter $c=0.2$ set experimentally, and this is the default
setting in \cite{vedaldi2010vlfeat}\footnote{The 0.2 clamping threshold is 'hard coded' into \cite{vedaldi2010vlfeat}.}.
This is followed by an additional normalization to ensure unit length
of the descriptor is preserved after thresholding. It is important
to note that the thresholding in Eq. \ref{eq: Clamping} maintains
invariance to affine contrast changes. The thresholding process, or
clamping, is thought to reduce the effects of camera saturation or
other illumination effects. However, we will show empirically in Sec.
\ref{subsec:Evaluation} that clamping also increases the general
matching performance of the descriptor, observed to be 14.4\% compared
to the performance without clamping on the Oxford dataset. This occurs
even when there exists consistent lighting conditions between image
pairs. The threshold parameter of $c=0.2$ is set rather arbitrarily;
and is fixed for every descriptor. By applying an automatic threshold
that is allowed to vary for every descriptor, we can significantly
improve the performance of the SIFT descriptor for image matching
problems. 

\section{Meaningful Clamping\label{sec:Meaningful-Clamping}}

The bins of the SIFT descriptor represent the underlying content of
a local image patch. We wish to detect when geometric structure is
present in the patch; and this is indicated by the observation of
large descriptor bin values. This amounts to detecting significant
bins by computing a perception threshold for each descriptor and using
that as the clamping limit. The idea is that once bins reach the perception
threshold, little information is gained by exceeding this value. A
contrario methodology is proposed to compute descriptor perception
thresholds, and is based on applying a mathematical foundation to
the concept of the Helmholtz principal, which states ``we immediately
perceive whatever could not happen by chance'' \cite{desolneux2007gestalt}.
It has been shown to be highly successful for many problems in image
processing such as detecting line segments \cite{von2010lsd}, change
detection \cite{flenner2011helmholtz}, contrasted boundaries \cite{desolneux2001edge},
vanishing points \cite{almansa2003vanishing}, and modes of histograms
\cite{delon2007nonparametric,flenner20082DHist}. 

Instead of trying to define a priori the structure of the underlying
image content, an impossible task for general natural images, we instead
define what it means to have a \emph{lack of structure}. Using the
Helmholtz principal, lack of structure is simply modeled as uniform
randomness, which we call the uniform background model, or the null
hypothesis $\mathcal{H}_{0}$. We assume the descriptor has been generated
from $\mathcal{H}_{0}$, and claim a detection, i.e. significant geometric
content is present, when there is a large deviation from $\mathcal{H}_{0}$.
If the observed event is extremely unlikely to have been generated
from this background model, we claim the event as \emph{meaningful}
because it could not have occurred by random chance.

Let $\Lambda$ be the histogram grid associated with the descriptor
${\bf d}$, which represents a set of $L=n(x)n(y)n(\theta)$ connected
bins such that every bin $\ell=\left(i,j,k\right)\in\Lambda$ contains
a number of sample counts ${\bf d}(\ell)$, and a neighborhood $\mathcal{C}_{\ell}\subset\Lambda$
of bins for which $\ell$ is connected. Introducing a neighborhood
set for each bin allows us to have circular connected angular histograms,
while spatial dimensions are rectangular. We also let $M=\sum_{\ell}{\bf d}\left(\ell\right)$
be the total number of samples of the descriptor and $p(\ell)$ be
the probability that a random sample is drawn in bin $\ell$, which
leads to the definition of the null hypothesis for the descriptor
${\bf d}$.
\begin{defn}
\label{def: background model}Let ${\bf d}$ be a SIFT descriptor
built on the grid $\Lambda$. The descriptor is said to be drawn from
the null hypothesis, $\mathcal{H}_{0}$, if every sample is independent,
identically, and uniformly distributed with $p(\ell)=\frac{1}{L}$
for every bin $\ell\in\Lambda$.
\end{defn}

It follows that the probability at least ${\bf d}(\ell)$ samples
are in bin $\ell$ under the null hypothesis, with $p(\ell)=1/L$,
is given by the binomial tail
\begin{align}
\mathbb{P}\left[k\geq{\bf d}\left(\ell\right)|\mathcal{H}_{0}\right]= & \mathcal{B}\left(M,{\bf d}(\ell),p\left(\ell\right)\right)\nonumber \\
= & \sum_{k={\bf d}\left(\ell\right)}^{M}\left(\begin{array}{c}
M\\
k
\end{array}\right)p(\ell)^{k}\left(1-p(\ell)\right)^{M-k}.\label{eq:Binomial Tail}
\end{align}
 When this probability becomes small, ${\bf d}\left(\ell\right)$
is unlikely to have occurred under the uniform background model, we
then reject the null hypothesis and conclude the bin $\ell$ must
be meaningful\emph{.} This results in detecting meaningful bins by
thresholding the probability in Eq. \ref{eq:Binomial Tail}. Given
the assumption that the data was drawn from the uniform background
model, we can compute for any bin $\ell$ the expected number of false
detections, denoted as NFA for the number of false alarms, as
\begin{equation}
\text{NFA}\left(\ell\right)=\mathcal{N}\mathcal{B}\left(M,{\bf d}(\ell),p(\ell)\right),\label{eq:NFA Definition}
\end{equation}
where $\mathcal{N}$ is the number of tests, and is typically defined
as the number of all possible connected subsets of the histogram.
$\mathcal{N}$ can be seen as a Bonferroni correction \cite{von2009computational,hochberg2009multiple}
for the expected value in Eq. \ref{eq:NFA Definition}. Which leads
to the following definition of a meaningful bin.
\begin{defn}
\label{def: meaningful bin}A bin $\ell\in\Lambda$ of the SIFT descriptor
${\bf d}$ is an $\varepsilon$-meaningful bin if
\[
\text{NFA}(\ell)=\mathcal{N}\mathcal{B}\left(M,{\bf d}(\ell),p(\ell)\right)<\varepsilon.
\]
\end{defn}

This leads to the question of what to use for $\varepsilon$? We can
follow the work of Desolneux, et al. \cite{desolneux2000meaningful},
and always set $\varepsilon=1$, since including the number of tests,
$\mathcal{N}$, allows the threshold to scale automatically with histogram
size. The setting of $\varepsilon=1$ can be interpreted as setting
the threshold so as to limit the expected number of false detections
under a uniform background model to less than one. This has two important
consequences. First, for some applications, it is important for the
algorithm to correctly give zero detections when no object exists.
Second, this strategy gives detection thresholds that are similar
to that of human perception \cite{desolneux2003computational}; and
the dependence on $\varepsilon$ is logarithmic and hence very weak
\cite{von2010lsd}. We will hereafter refer to an $\varepsilon$-meaningful
bin as just a meaningful bin.

We can now select a clamping threshold for ${\bf d}$ as the minimum
descriptor bin value needed to be detected as a meaningful bin. For
a given descriptor ${\bf d}$, with corresponding properties $M$
and $p\left(\ell\right)=1/L$, we define this threshold as
\begin{equation}
t_{{\bf d}}=\text{min}\left\{ k:\mathcal{N}\mathcal{B}\left(M,k,p\left(\ell\right)\right)<1\right\} .\label{eq: Detection Threshold}
\end{equation}
We then proceed to create the new clamped descriptor as 

\begin{equation}
{\bf d}_{t}\left(\ell\right)=\text{min}\left(t_{{\bf d}},{\bf d}\left(\ell\right)\right),\label{eq: Clamped descriptor}
\end{equation}
for every bin $\ell\in\Lambda$.

\section{Implementation Details\label{sec:Implementation-Details}}

The a contrario threshold in Eq. \ref{eq: Detection Threshold} has
dependence on $\mathcal{N}$, which is defined as the number of all
possible connected subsets of $\Lambda$. However, for any histogram
greater than dimension one, we cannot explicitly compute this, and
instead use the number of aligned rectangular regions
\begin{equation}
\mathcal{N}_{\text{Rect}}=\frac{1}{8}n(x)n(y)n(\theta)\left(n(x)+1\right)\left(n(y)+1\right)\left(n(\theta)+1\right).\label{eq:Number of aligned rectangular regions}
\end{equation}
$\mathcal{N}_{\text{Rect}}$ represents a (loose) lower bound of $\mathcal{N}$.
There could also be concern with respect to computing the inverse
binomial tail in Eq. \ref{eq: Detection Threshold}. While efficient
computational libraries exist to directly calculate the detection
threshold\footnote{For example the quantile function in the binomial library of Boost. },
this still requires an iterative method since no closed form solution
exits. This may be undesirable for certain real-time applications.
We can instead create an approximation to Eq. \ref{eq: Detection Threshold}
by applying the bound 
\begin{equation}
-\frac{1}{M}\text{ln}\,\mathbb{P}\left[{\bf d}\left(\ell\right)\geq rM|\mathcal{H}_{0}\right]\leq\frac{\left(r-p\right)^{2}}{p\left(1-p\right)}+O\left(\frac{\text{ln}\,M}{M}\right),\label{eq: Slud bound}
\end{equation}
with $r=k/M$ and $p=1/L$ \cite{desolneux2007gestalt}. The bound
in Eq. \ref{eq: Slud bound} is valid when either (a) $p\leq1/4$
and $p\leq r$, or (b) $p\leq r\leq1-p$ \cite{slud1977distribution}.
As $M$ grows large, the $O\left(\frac{\text{ln}\,M}{M}\right)$ term
becomes small\footnote{For any typical implementation of SIFT, the $O\left(\frac{\text{ln}\,M}{M}\right)$
term is negligible and the bound $-\frac{1}{M}\text{ln}\,\mathbb{P}\left[{\bf d}\left(\ell\right)\geq rM|\mathcal{H}_{0}\right]\leq\frac{\left(r-p\right)^{2}}{p\left(1-p\right)}$
is valid.} and Eq. \ref{eq: Slud bound} converges to the central limit approximation.
Using this we can solve for the detection threshold as 
\begin{equation}
\tilde{t}_{{\bf d}}=Mp+\alpha\left(\mathcal{N}_{\text{Rect}}\right)\sqrt{Mp\left(1-p\right)},\label{eq: Bounded bin threshold}
\end{equation}
with $\alpha\left(\mathcal{N}_{\text{Rect}}\right)=\sqrt{-\text{ln}\left(1/\mathcal{N}_{\text{Rect}}\right)}$.
From this we can compute a new clamped descriptor, ${\bf d}_{\tilde{t}}\left(\ell\right)$,
with Eq. \ref{eq: Clamped descriptor} using the bin threshold $\tilde{t}_{{\bf d}}$
of Eq. \ref{eq: Bounded bin threshold}. Using the approximation $\tilde{t}_{{\bf d}}$
still maintains the property from Eq. \ref{eq: Clamped descriptor}
that ${\bf d}_{\tilde{t}}\left(\ell\right)\leq t_{{\bf d}}$.
\begin{prop}
\label{prop: Approx clamping}Let ${\bf d}_{\tilde{t}}$ be a SIFT
descriptor clamped with the approximate threshold $\tilde{t}_{{\bf d}}$
given in Eq. \ref{eq: Bounded bin threshold}, and $t_{{\bf d}}$
is the exact threshold given in Eq. \ref{eq: Detection Threshold}.
Then, as $M$ grows large, ${\bf d}_{\tilde{t}}\left(\ell\right)\leq t_{{\bf d}}$
for all $\ell\in\Lambda$ such that either (a) $p\leq1/4$ and $p\leq\frac{\tilde{t}_{{\bf d}}}{M}$,
or (b) $p\leq\frac{\tilde{t}_{{\bf d}}}{M}\leq1-p$. 
\end{prop}

\begin{proof}
Since $\mathcal{N}_{\text{Rect}}$ is a lower bound on the true number
of tests, $\mathcal{N}$, then $\mathcal{N}_{\text{Rect}}\mathcal{B}\left(M,k,p\left(\ell\right)\right)\leq\mathcal{N}\mathcal{B}\left(M,k,p\left(\ell\right)\right)$
which implies that
\begin{align*}
 & \text{min}\left\{ \bar{k}:\mathcal{N}_{\text{Rect}}\mathcal{B}\left(M,\bar{k},p\left(\ell\right)\right)<1\right\} \\
 & \leq\text{min}\left\{ k:\mathcal{N}\mathcal{B}\left(M,k,p\left(\ell\right)\right)<1\right\} ,
\end{align*}
 and hence $\bar{k}\leq k=t_{{\bf d}}$. From Eq. \ref{eq: Slud bound}
we have $\tilde{t}_{{\bf d}}\leq\bar{k}$, which implies $\tilde{t}_{{\bf d}}\leq t_{{\bf d}}$.
The result follows from Eq. \ref{eq: Clamped descriptor} with ${\bf d}_{\tilde{t}}\left(\ell\right)\leq\tilde{t}_{{\bf d}}\leq t_{{\bf d}}$,
for every bin $\ell\in\Lambda$. 
\end{proof}
The significance of Proposition \ref{prop: Approx clamping} is that
we can safely use Eq. \ref{eq: Bounded bin threshold} and ensure
the descriptor is appropriately clamped without having to determine
the true number of tests, $\mathcal{N}$, or iterate to find the inverse
of the binomial tail. Conditions (a), (b), and the requirement that
$M$ is sufficiently large in Eq. \ref{eq: Slud bound} are very weak
since for any practical implementation of the SIFT descriptor, these
conditions are met. 

\section{Results\label{sec:Results}}

We present image matching results applying the newly developed meaningful
clamping method, and compare it to the clamping procedure proposed
by Lowe. For reference, we also compare both clamping methods to descriptors
with which no clamping was performed.

\subsection{Dataset\label{subsec:Dataset}}

To evaluate matching performance, we use the Oxford dataset \cite{mikolajczyk2005comparison},
which is comprised of 40 image pairs of various scene types undergoing
different camera poses and transformations. These include viewpoint
angle, zoom, rotation, blurring, compression, and illumination. The
set contains eight categories, each of which consists of image pairs
undergoing increasing magnitudes of transformations. Included with
each image pair is a homography matrix, which represents the ground
truth mapping of points between the images. The transformations applied
to the images are real and not synthesized as in \cite{fischer2014descriptor}.
The viewpoint and zoom+rotation categories are generated by focal
length adjustments and physical movement of the camera. Blur is generated
by varying the focus of the camera; and illumination by varying the
aperture. The compression set was created by applying JPEG compression
and adjusting the image quality parameter. 

\subsection{Metrics\label{subsec:Metrics}}

To evaluate the performance of local descriptors with respect to image
matching, we follow the methods of \cite{mikolajczyk2005performance}.
Given a pair of images we extract SIFT features from both images.
A match between two descriptors is determined when the Euclidean distance
is less than some threshold $t$. Any descriptor match is considered
a correct match if the two detected features correspond as defined
in \cite{mikolajczyk2005comparison}. Using the ground truth homography
mapping supplied with the dataset, features are considered to correspond
when the area of intersection over union is greater than 50\% to be
consistent with \cite{mikolajczyk2005performance}. For some value
of $t$ we can compute recall as
\[
\text{recall}(t)=\frac{\text{\# correct matches}\,(t)}{\text{\# correspondences}},
\]
as well as 1-precision
\[
1-\text{precision}(t)=\frac{\text{\# false matches}\,(t)}{\text{\# correct matches}\,(t)+\text{\# false matches}\,(t)}.
\]
The pair $\left(\text{recall}(t),1-\text{precision}(t)\right)$ represents
a point in space; and by varying $t$ we can create curves that demonstrate
the matching performance of the descriptor. This is called the \emph{precision
recall curve}; and we follow the method of \cite{everingham2010pascal}
to compute the area under the curve, producing a value called \emph{average
precision} (AP)\footnote{We use 100 points to sample the precision recall curves as opposed
to 11 proposed in \cite{everingham2010pascal}. This gives a higher
resolution estimate of the AP.}. Larger AP indicates superior matching performance. The average of
APs, across individual categories or the entire dataset, provides
the \emph{mean average precision} (mAP) used to compare clamping methods.

\subsection{Evaluation\label{subsec:Evaluation}}

\begin{table}
\begin{centering}
\begin{tabular}{|c|c|c|c|}
\hline 
Category & No Clamping & Lowe Clamping & MC\tabularnewline
\hline 
\hline 
Graffiti & 0.123 & 0.161 & 0.205\tabularnewline
\hline 
Wall & 0.327 & 0.371 & 0.405\tabularnewline
\hline 
Boats & 0.301 & 0.341 & 0.375\tabularnewline
\hline 
Bark & 0.111 & 0.119 & 0.120\tabularnewline
\hline 
Trees & 0.207 & 0.288 & 0.366\tabularnewline
\hline 
Bikes & 0.414 & 0.371 & 0.496\tabularnewline
\hline 
Leuven & 0.387 & 0.538 & 0.635\tabularnewline
\hline 
UBC & 0.558 & 0.588 & 0.615\tabularnewline
\hline 
All images & 0.303 & 0.347 & 0.402\tabularnewline
\hline 
\end{tabular}
\par\end{centering}
\caption{Mean average precision for each category of the Oxford dataset. SIFT
detector parameter FirstOctave is set to 0.\label{tab: Results FO - 0}}
\end{table}

We compute the AP for every image pair in the Oxford dataset, each
for two different parameter settings of the SIFT detector. This parameter
is called FirstOctave, and we test for both 0 and -1. Setting FirstOctave
to -1 upsamples the image before creating the scale space, generating
a great deal more features than with 0, resulting in more total matches,
but with lower overall AP. It is important to test for this setting
because it allows for greater scale variations between images, and
is the default setting for SIFT in the Covariant Features toolbox
in the popular VLFeat open-source library \cite{vedaldi2010vlfeat}.
It also shows how clamping impacts performance in large sets of SIFT
points, and indicates how well the method scales with large amounts
of data. For certain image pairs, the distortion between images is
great enough, that little or no feature correspondences exist. Under
these circumstances, no matches are found, and we cannot compute the
precision recall curves. We define the AP to be zero in this case.

Table \ref{tab: Results FO - 0} compares the mAP for each category
in the Oxford dataset when the SIFT FirstOctave is set to 0. MC systematically
outperforms Lowe clamping for every image transform type. It also
shows that clamping can improve matching performance in general image
pairs, not just in cases of significant illumination differences.
The leuven category of lighting shows an impressive 18.2\% improvement,
but does not exhibit the greatest gain, which occurred in bikes (blur)
at 33.6\%. The method shows remarkable performance on blurred images,
with trees improving 27.0\%. The bark (zoom+rotation) had the least
improvement at 1.4\%. However, it should be noted that it could be
an artifact of the SIFT detector which extracted few correct correspondences
for this category. Boats, which also varied zoom+rotation, had a 9.9\%
increase. The mean AP for all image pairs of the Oxford dataset improved
by 15.9\% compared to Lowe clamping. Fig. \ref{fig: AP Scatter Plot}
shows a direct comparison between clamping methods, with each point
representing the AP of an image pair. 

\begin{table}
\begin{centering}
\begin{tabular}{|c|c|c|c|}
\hline 
Category & No Clamping & Lowe Clamping & MC\tabularnewline
\hline 
\hline 
Graffiti & 0.016 & 0.035 & 0.110\tabularnewline
\hline 
Wall & 0.230 & 0.270 & 0.320\tabularnewline
\hline 
Boats & 0.054 & 0.118 & 0.244\tabularnewline
\hline 
Bark & 0.049 & 0.063 & 0.068\tabularnewline
\hline 
Trees & 0.043 & 0.096 & 0.173\tabularnewline
\hline 
Bikes & 0.141 & 0.112 & 0.185\tabularnewline
\hline 
Leuven & 0.115 & 0.210 & 0.365\tabularnewline
\hline 
UBC & 0.215 & 0.305 & 0.411\tabularnewline
\hline 
All images & 0.108 & 0.152 & 0.234\tabularnewline
\hline 
\end{tabular}
\par\end{centering}
\caption{Mean average precision for each category of the Oxford dataset. SIFT
detector parameter FirstOctave is set to -1.}

\end{table}

For large scale experiments with the FirstOctave parameter set to
-1, the performance jumps dramatically, and shows that the improvement
in matching increases as the number of points increases. The category
exhibiting the most improvement was graffiti (viewpoint) with a remarkable
215.2\% increase. Again, bark had the least improvement with 7.9\%.
Even with the FirstOctave parameter set to -1, the SIFT detector performed
poorly on the bark category and generated few correspondences, influencing
the matching results as before. As a reference, boats increased by
106.9\%. The mean AP increased by 54.0\% for all image pairs in the
dataset.

It is important to note that while SIFT is used as the detector for
this experiment, other detectors may be used and obtain similar results.
However, much like the SIFT detector, there exist other fundamental
parameters that may greatly influence the number of total points generated.
Experiments point to the number of detected points generated as the
single largest factor relating the amount of improvement over Lowe
clamping. The remarkable property observed in the experiments listed
above is that with a larger amount of detected points to match, the
percentage improvement in AP \emph{increases}. Also of interest, is
that clamping can augment other recent advances in image descriptor
construction, for example DSP-SIFT \cite{dongS15}. 

\begin{figure}
\begin{centering}
\includegraphics[width=8cm]{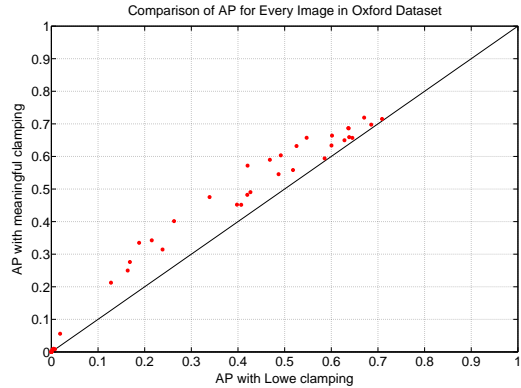}
\par\end{centering}
\caption{The AP of every image pair is represented by a red dot. The x-axis
value is the AP for the pair with Lowe clamping, and the y-axis is
the AP for the same pair with meaningful clamping. The black line
is added for reference. Any point above the line represents an image
pair in the Oxford dataset, such that meaningful clamping increases
AP matching performance.\label{fig: AP Scatter Plot}}

\end{figure}

\section{Conclusions and Future Work\label{sec:Conclusions-and-Future}}

A new method to threshold SIFT descriptors was presented. This method
significantly improves mAP for image matching on the standard Oxford
dataset. Future work is to study the impact meaningful clamping has
on other problems, such as large scale image retrieval. Also of interest
is the study of \emph{why} meaningful clamping (and also clamping
in general) has such a large impact on image matching.

The author conjectures that clamping effects the distribution of large
point sets of descriptors. If the descriptor is not clamped, then
a small number of descriptor bins would dominate all other bins. This
would constrain the points to lie mostly along the axes. Performing
nearest neighbor-type searches could become ambiguous, since many
points would exist with a similar spatial distance. By clamping, we
are thresholding bin magnitudes; and this causes the points to 'spread
out', and more uniformly occupy the $\mathbb{R}_{+}^{L}$ space in
which the descriptors lie. This conjecture is supported by the observation
in the presented experiments that the improvement drastically increased
when attempting to match larger sets of points extracted from the
image pairs when the SIFT detector parameter FirstOctave was changed
to -1, generating many more features for each image. 

\section*{Acknowledgments}

The author would like to thank UCLA professor Stefano Soatto for many
great discussions on the topic of clamping SIFT descriptors, and his
student Jingming Dong for direction on descriptor evaluation. The
author would also like to thank Lawrence Peterson, Shawnasie Kirchner,
and the anonymous reviewers for spending valuable time to provide
a careful and thorough review of this paper.

\bibliographystyle{plain}
\nocite{*}
\bibliography{ICIPRef}

\begin{thebibliography}{10}

\bibitem{almansa2003vanishing}
A.~Almansa, A.~Desolneux, and S.~Vamech.
\newblock Vanishing point detection without any a priori information.
\newblock {\em IEEE Transactions on Pattern Analysis and Machine Intelligence},
  25(4):502--507, 2003.

\bibitem{bay2006surf}
H.~Bay, T.~Tuytelaars, and L.~Van Gool.
\newblock {SURF}: Speeded up robust features.
\newblock In {\em European Conference on Computer Vision (ECCV)}, pages
  404--417. Springer, 2006.

\bibitem{berger2007EffectiveComponentTree}
C.~Berger, T.~G{\'e}raud, R.~Levillain, N.~Widynski, A.~Baillard, and
  E.~Bertin.
\newblock Effective component tree computation with application to pattern
  recognition in astronomical imaging.
\newblock In {\em IEEE International Conference on Image Processing (ICIP)},
  2007.

\bibitem{cao2008theory}
F.~Cao, J.-L. Lisani, J.-M. Morel, P.~Mus{\'e}, and F.~Sur.
\newblock {\em A Theory of Shape Identification}, volume 1948.
\newblock Springer, 2008.

\bibitem{delon2004histogram}
J.~Delon, A.~Desolneux, J.-L. Lisani, and A.~B. Petro.
\newblock Histogram analysis and its applications to fast camera stabilization.
\newblock In {\em International Workshop on Systems, Signals and Image
  Processing}, pages 431--434, 2004.

\bibitem{delon2007nonparametric}
J.~Delon, A.~Desolneux, J.-L. Lisani, and A.~B. Petro.
\newblock A nonparametric approach for histogram segmentation.
\newblock {\em IEEE Transactions on Image Processing}, 16(1):253--261, 2007.

\bibitem{desolneux2000meaningful}
A.~Desolneux, L.~Moisan, and J.-M. Morel.
\newblock Meaningful alignments.
\newblock {\em International Journal of Computer Vision}, 40(1):7--23, 2000.

\bibitem{desolneux2001edge}
A.~Desolneux, L.~Moisan, and J.-M. Morel.
\newblock Edge detection by {Helmholtz} principle.
\newblock {\em Journal of Mathematical Imaging and Vision}, 14(3):271--284,
  2001.

\bibitem{desolneux2003computational}
A.~Desolneux, L.~Moisan, and J.-M. Morel.
\newblock Computational gestalts and perception thresholds.
\newblock {\em Journal of Physiology-Paris}, 97(2):311--324, 2003.

\bibitem{desolneux2007gestalt}
A.~Desolneux, L.~Moisan, and J.-M. Morel.
\newblock {\em From Gestalt Theory to Image Analysis: A Probabilistic
  Approach}, volume~34.
\newblock Springer, 2007.

\bibitem{dongS15}
J.~Dong and S.~Soatto.
\newblock Domain size pooling in local descriptors: {DSP-SIFT}.
\newblock In {\em Proceedings of the IEEE Conference on Computer Vision and
  Pattern Recognition (CVPR)}. June 2015.

\bibitem{everingham2010pascal}
M.~Everingham, L.~Van Gool, C.~Williams, J.~Winn, and A.~Zisserman.
\newblock The {PASCAL} visual object classes ({VOC}) challenge.
\newblock {\em International Journal of Computer Vision}, 88(2):303--338, 2010.

\bibitem{fischer2014descriptor}
P.~Fischer, A.~Dosovitskiy, and T.~Brox.
\newblock Descriptor matching with convolutional neural networks: A comparison
  to {SIFT}.
\newblock {\em arXiv preprint arXiv:1405.5769}, 2014.

\bibitem{flenner2011helmholtz}
A.~Flenner and G.~Hewer.
\newblock A {Helmholtz} principle approach to parameter free change detection
  and coherent motion using exchangeable random variables.
\newblock {\em SIAM Journal on Imaging Sciences}, 4(1):243--276, 2011.

\bibitem{flenner20082DHist}
A.~Flenner, G.~Hewer, and C.~Kenney.
\newblock Two dimensional histogram analysis using the {Helmholtz} principle.
\newblock {\em Inverse Problems and Imaging}, 2(4):485--525, 2008.

\bibitem{forstner2009detecting}
W.~Forstner, T.~Dickscheid, and F.~Schindler.
\newblock Detecting interpretable and accurate scale-invariant keypoints.
\newblock In {\em IEEE International Conference on Computer Vision (ICCV)},
  pages 2256--2263. IEEE, 2009.

\bibitem{von2009computational}
R.~G.~Von Gioi and J.~Jakubowicz.
\newblock On computational {Gestalt} detection thresholds.
\newblock {\em Journal of Physiology-Paris}, 103(1):4--17, 2009.

\bibitem{von2010lsd}
R.~G.~Von Gioi, J.~Jakubowicz, J.-M. Morel, and G.~Randall.
\newblock {LSD}: A fast line segment detector with a false detection control.
\newblock {\em IEEE Transactions on Pattern Analysis and Machine Intelligence},
  32(4):722--732, 2010.

\bibitem{hochberg2009multiple}
Y.~Hochberg and A.~C. Tamhane.
\newblock {\em Multiple Comparison Procedures}.
\newblock John Wiley \& Sons, New York, 1987.

\bibitem{ke2004pca}
Y.~Ke and R.~Sukthankar.
\newblock {PCA-SIFT}: A more distinctive representation for local image
  descriptors.
\newblock In {\em Proc. of the IEEE Conference on Computer Vision and Pattern
  Recognition (CVPR)}, volume~2, pages II:506--513. IEEE, 2004.

\bibitem{lowe2004distinctive}
D.~G. Lowe.
\newblock Distinctive image features from scale-invariant keypoints.
\newblock {\em International Journal of Computer Vision}, 60(2):91--110, 2004.

\bibitem{matas2004robust}
J.~Matas, O.~Chum, M.~Urban, and T.~Pajdla{\'a}s.
\newblock Robust wide-baseline stereo from maximally stable extremal regions.
\newblock {\em Image and Vision Computing}, 22(10):761--767, 2004.

\bibitem{mikolajczyk2004scale}
K.~Mikolajczyk and C.~Schmid.
\newblock Scale \& affine invariant interest point detectors.
\newblock {\em International Journal of Computer Vision}, 60(1):63--86, 2004.

\bibitem{mikolajczyk2005performance}
K.~Mikolajczyk and C.~Schmid.
\newblock A performance evaluation of local descriptors.
\newblock {\em IEEE Transactions on Pattern Analysis and Machine Intelligence},
  27(10):1615--1630, 2005.

\bibitem{mikolajczyk2005comparison}
K.~Mikolajczyk, T.~Tuytelaars, C.~Schmid, A.~Zisserman, J.~Matas,
  F.~Schaffalitzky, T.~Kadir, and L.~Van Gool.
\newblock A comparison of affine region detectors.
\newblock {\em International Journal of Computer Vision}, 65(1-2):43--72, 2005.

\bibitem{najman2006BuildingComponentTree}
L.~Najman and M.~Couprie.
\newblock Building the component tree in quasi-linear time.
\newblock {\em IEEE Transactions on Image Processing}, 15(11):3531--3539, 2006.

\bibitem{otero2014anatomy}
I.~R. Otero and M.~Delbracio.
\newblock Anatomy of the {SIFT} method.
\newblock {\em Image Processing On Line}, 4:370--396, 2014.

\bibitem{rabin2009approches}
J.~Rabin.
\newblock {\em Approches robustes pour la comparaison d'images et la
  reconnaissance d'objets}.
\newblock PhD thesis, T{\'e}l{\'e}com ParisTech, 2009.

\bibitem{rabin2009statistical}
J.~Rabin, J.~Delon, and Y.~Gousseau.
\newblock A statistical approach to the matching of local features.
\newblock {\em SIAM Journal on Imaging Sciences}, 2(3):931--958, 2009.

\bibitem{slud1977distribution}
E.~V. Slud.
\newblock Distribution inequalities for the binomial law.
\newblock {\em The Annals of Probability}, pages 404--412, 1977.

\bibitem{vedaldi2010vlfeat}
A.~Vedaldi and B.~Fulkerson.
\newblock {VLFeat}: An open and portable library of computer vision algorithms.
\newblock In {\em Proceedings of the International Conference on Multimedia},
  pages 1469--1472. ACM, 2010.

\end{thebibliography}

\end{document}